\documentclass[sigplan,10pt]{acmart}

\AtBeginDocument{%
  }

\setcopyright{acmlicensed}

\copyrightyear{2025}
\acmYear{2025}
\setcopyright{cc}
\setcctype{by}
\acmConference[PACMI '25]{Practical Adoption Challenges of ML for
Systems}{October 13--16, 2025}{Seoul, Republic of Korea}
\acmBooktitle{Practical Adoption Challenges of ML for Systems (PACMI '25),
October 13--16, 2025, Seoul, Republic of
Korea}\acmDOI{10.1145/3766882.3767180}
\acmISBN{979-8-4007-2205-9/2025/10}

\usepackage{algorithmicx}
\usepackage{algpseudocode}
\usepackage{listings}
\usepackage{empheq}
\usepackage{multicol}
\usepackage{caption}
\usepackage{subcaption}
\usepackage{bbm}

\usepackage{booktabs}
\usepackage{graphicx}
\usepackage{tikz}
\usepackage{pgfplots}
\pgfplotsset{width=13cm,compat=1.9}
\usetikzlibrary{calc}
\usepackage{verbatim}
\usepackage{mathtools} 
\usepackage{algorithm}
\usepackage{float}
\usepackage{units}

\newcommand\ci{\perp\!\!\!\perp}

\usepackage{amsthm}

\theoremstyle{remark}

\theoremstyle{definition}

\usepackage{amsmath}

\begin{document}

\title{FLOSS: Federated Learning with Opt-Out \\and Straggler Support} 



\author{David J. Goetze}
\email{david.goetze@gmail.com}
\affiliation{%
  \institution{Williams College, USA}
  \country{}
}

\author{Dahlia J.~Felten}
\email{df15@williams.edu}
\affiliation{%
  \institution{Williams College, USA}
  \country{}
}

\author{Jeannie R.~Albrecht}
\email{jra1@williams.edu}
\affiliation{%
  \institution{Williams College, USA}
 \country{}
 }

\author{Rohit Bhattacharya}
\email{rb17@williams.edu}
\affiliation{%
  \institution{Williams College, USA}
  \country{}
  }

\renewcommand{\shortauthors}{Goetze, Felten, Albrecht, and Bhattacharya}

\begin{abstract}
   Previous work on data privacy in federated learning systems focuses on privacy-preserving operations for data from users who have agreed to share their data for training. However, modern data privacy agreements also empower users to use the system while opting out of sharing their data as desired. When combined with stragglers that arise from heterogeneous device capabilities, the result is missing data from a variety of sources that introduces bias and degrades model performance. In this paper, we present FLOSS, a system that mitigates the impacts of such missing data on federated learning in the presence of stragglers and user opt-out, and empirically demonstrate its performance in simulations. 
\end{abstract}







\settopmatter{printfolios=true} 
\maketitle
\pagestyle{empty}
\thispagestyle{empty}

\section{Introduction}
\label{sec:intro}





Federated learning (FL) is a privacy-preserving form of machine learning in which a model is trained across a distributed set of clients, eliminating the need for individual users to share their data with a central server~\cite{mcmahan17}. Instead, each participant trains a local model and sends only weights or gradients back to the server.  The server aggregates these to update the central model and broadcasts it back to the clients in each training round.  Since sensitive data are not sent to the server, this approach helps mitigate privacy risks associated with centralized data storage and transfer.

While FL systems offer advantages with respect to privacy, their distributed nature introduces several challenges related to missing data. Some gradients may be lost or delayed due to problems with the devices or network.  The presence of these {\it stragglers} in distributed computing is a well-studied problem~\cite{mapreduce04, albrecht06} that causes missing data in FL systems~\cite{oort21, refl23, fedprox20}.  However, beyond just infrastructure-level connectivity issues, users of FL systems may also decide to opt out of gradient/weight sharing for increased privacy. Modern data privacy agreements give users the ability to change their mind and opt in or opt out as desired at any point during training.  As shown in Figure~\ref{fig:teaser}, some participants may elect to withhold their data, thus preventing the central model from using it in model updates.  When this occurs, the model may become biased and its accuracy may suffer as a result.  

\begin{figure*}[t]
  \centering
    \begin{subfigure}[height=1.5in]{0.24\textwidth}
    \centering
    \includegraphics[scale=0.4]{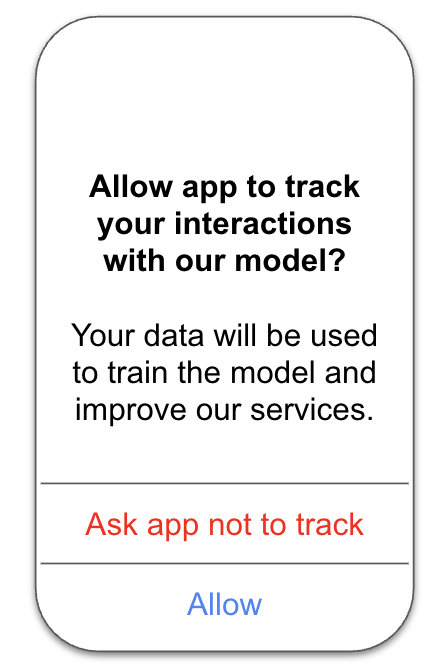}
    \caption{\vspace{0em}}
    \end{subfigure}
    \begin{subfigure}[height=1.5in]{0.24\textwidth}
    \centering
    \includegraphics[scale=0.4]{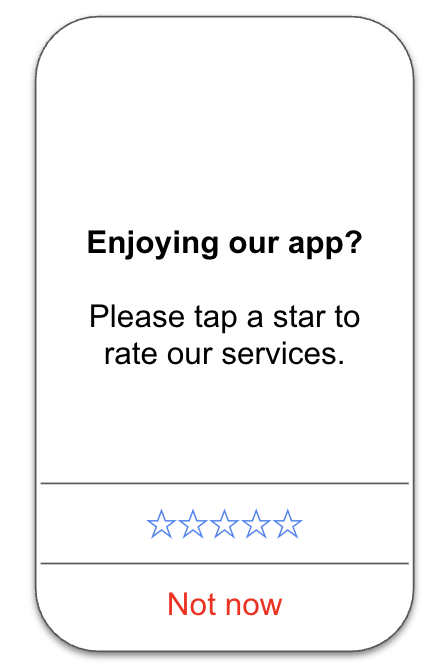}
    \caption{\vspace{0em}}
    \end{subfigure}
    \begin{subfigure}[height=1.5in]{0.5\textwidth}
    \centering
    \includegraphics[scale=0.36]{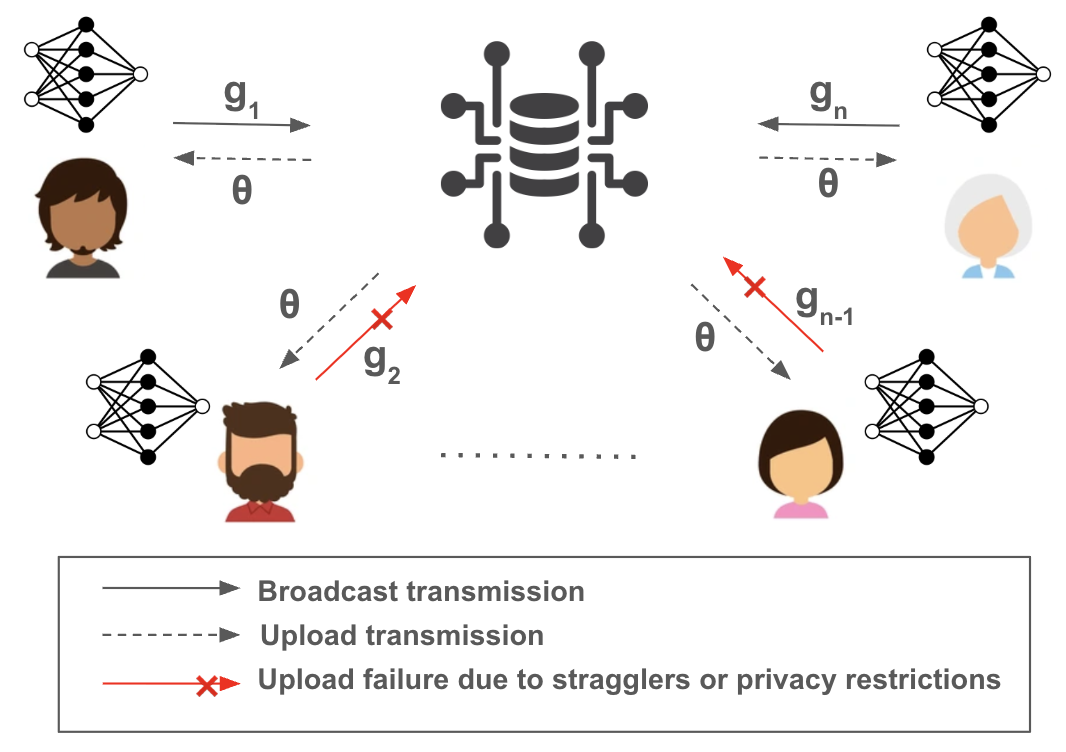}
    \caption{\vspace{0em}}
    \end{subfigure}
  \caption{\small (a) A prompt allowing users to opt out of training, and (b) a prompt asking for user feedback. In addition to stragglers, either of these can lead to missing data when users select the red option. (c) This affects the FL system, as some gradients $g_i$ are systematically missing, introducing bias into the model's weights ($\theta$) at each update step.}
  \Description{}
  \label{fig:teaser}
\end{figure*}

Although model training is typically robust to  {\it missing completely at random (MCAR)} data, {\it missing at random (MAR)} and {\it missing not at random (MNAR)} data are more problematic.  MAR in the FL context implies that the likelihood of user data being excluded is not related to the missing data itself, but still systematically different based on observable device or network properties. For example, straggling participants in rural areas with poor network connectivity may be excluded from training. MNAR implies that the tendency for data to be excluded is related to the missing data itself. For instance, participants from a specific demographic class who possess data not represented elsewhere may opt out of training.  


Thus, it is generally not safe to assume that data are MCAR in FL systems, and the selection bias from MAR and MNAR data can negatively impact the performance of the model. Our work aims to address this problem.  Specifically, we leverage 
modern theory in inverse probability weighting (IPW)~\cite{horvitz1952generalization, seaman2013review} and missing data graphical models~\cite{mohan2021graphical, miao2024identification, nabi2025sinica} in order to reweight the gradient aggregation in FL systems and mitigate the impacts of MCAR, MAR, and MNAR missing data while preserving user privacy, thereby improving the overall usability of FL for practical applications.  

To this end, we present \textbf{FLOSS}: a privacy-preserving FL system for mitigating the impacts of missing data without forcing additional data collection or violating user data-sharing agreements.  We present a formal model of missing data in FL systems, and describe how we support opt-out user privacy policies using reweighted selection.  We also provide preliminary results from a prototype implementation that evaluates our ability to correct for missing data.  

\section{Notation and Problem Setup}
\label{sec:notation}

We set up the notation used in our paper as follows.
\begin{align*}
&X : \text{a set of features used for generating predictions}. \\
&Y: \text{the outcome of interest (real-valued or categorical)}. \\
&D: \text{user info collected at sign-up---{\it e.g.}, age + device specs}. \\
&S: \text{user satisfaction with system and model performance}. \\
&R: \text{binary indicator of responsiveness to server requests} \\
&\text{\hspace{0.4cm} $R=0$ for stragglers/users opting out; $R=1$ otherwise}.
\end{align*}

In a typical FL setup, we have a set of $n$ users $\mathcal U$, each with their own private dataset consisting of multiple realizations of the features $X$ and outcome $Y$. A model $h_\theta: x \mapsto y$ is then trained in a decentralized manner to minimize the expected loss $E[L(x,y, \theta)]$, often approximated by the empirical risk $\frac{1}{n} \sum_{i=1}^n L(x_i, y_i, \theta)$, over multiple rounds. Moving forward, we suppress dependence of the loss and gradient functions on the data $x, y$ for brevity. At each step $t$, the central server samples a subset of users $\mathcal U'$ of size $k$, and requests  gradients $G(\theta^{(t)})$ of the loss function evaluated over their private datasets. Each sampled  device uploads their gradients $g_1(\theta^{(t)}), \dots, g_k(\theta^{(t)})$, or noisy and clipped versions of them to add differential privacy \cite{abadi2016deep, birrell2024differentially, das2025security}. The central server then aggregates 
these gradients to obtain $\overline{g}(\theta^{(t)})$ and updates the model as $\theta^{(t+1)} \gets \theta^{(t)} - \eta \cdot \overline{g}(\theta^{(t)})$, where $\eta$ is the learning rate. Note we assume equal-sized datasets (i.e., $X$ and $Y$ are understood to contain an equal number of realizations of features and outcomes for each user) for brevity of notation, but our methodology generalizes in a straightforward manner. Finally, the updated model $h_{\theta^{(t+1)}}$ is  broadcast to all users $u \in {\mathcal U}$, and the process repeats.

Two key issues arise in the above process that lead to missing data: (i) Some devices, known as stragglers \cite{albrecht06, mapreduce04},  may fail to upload their gradients in a reasonable time frame and the server is forced to perform the aggregation step without them, and (ii) certain users may decline to share their data for training as part of a data-sharing agreement, so not all devices can be prompted for their gradients. Both of these issues can lead to \emph{systematic}, rather than completely random, missingness of gradients.

{When the data are systematically missing, performing gradient aggregation based on just the observed data leads to biased estimates relative to the aggregate that would be obtained had everyone been able and willing to share their data (possibly contrary to fact). That is, we use the term \emph{bias} here in a statistical sense---a tendency of parameter and gradient estimates computed from just observed data to deviate from the values that would be obtained had there been no missing data at all. This kind of bias arising due to missing data (and the closely related problem of selection bias) is well documented in the causal inference literature, as it often leads to biased estimation of causal effects \cite{mohan2013missing, nabi2025sinica}. On the other hand, while there has been preliminary work on missing data in FL systems when the data are missing completely at random \cite{valdeira2025vertical}, to our knowledge, the issues that arise in FL systems due to systematic missingness---and how to mitigate such issues---have not been studied thus far. We note also that bias due to missing data and bias in the sense of whether an algorithm is fair (or not) are, in fact, related problems; it has also been shown that mitigating the effects of missing data can also lead to improved fairness \cite{kallus2018residual, goel2021importance}. Thus, in this work we focus our attention on mitigating bias due to missing data, with the understanding that this can, in some cases, also lead to improved performance on certain measures of algorithmic fairness.} 

In the following, we formalize the kind of data-sharing agreements we support in FLOSS, which reflects mandatory user privacy agreements that are ubiquitous across modern machine learning applications \cite{das2025security}.

\paragraph{Data-sharing agreement} For all users, the actual values of the features and outcomes present in their individual datasets are never shared with the central server, {\it i.e.}, the fine-grained data are always private. Further, if a user opts out of collaborative training of the model $h_\theta$, then any outputs obtained by running this model on their data will also not be shared with the central server. This includes coarse-grained outputs of the model, such as losses and gradients.


\begin{figure*}[t]
    \centering
    \scalebox{0.8}{
 
        \begin{tikzpicture}[>=stealth, node distance=1.5cm]
        \tikzstyle{square} = [draw, thick, minimum size=1.0mm, inner sep=3pt]
        \begin{scope}
            \path[->, very thick]
            node[] (d) {$D$}
            node[right of=d] (x) {\color{red} $X$}
            node[right of=x] (y) {\color{red}  $Y$}
            node[right of=y, xshift=0.5cm] (l) {\color{red}  $G(\theta)^{miss}$}
            node[right of=l] (r) {$R$}
            node[below of=l, yshift=0.5cm] (s) {\color{white} $S^{miss}$}
            node[below of=y] (label) {\textbf{(a)}}

            
            (x) edge[] (y)
            (x) edge[double, blue, bend left] (l)
            (y) edge[double, blue] (l)

            (x) edge[bend right=40] (r)
            (y) edge[bend right=25] (r)
            (d) edge[] (x)
            (d) edge[bend left=25] (y)
            (d) edge[bend left=35] (r)
            
            ;
        \end{scope}
    

 
        \begin{scope}[xshift=10cm]
            \path[->, very thick]
            node[] (d) {$D'$}
            node[right of=d] (z) {$Z$}
            node[right of=z] (x) {\color{red} $X$}
            node[right of=x] (y) {\color{red} $Y$}
            node[right of=y, xshift=0.5cm] (l) {\color{red} $G(\theta)^{miss}$}
            node[right of=l] (r) {$R$}
            node[below of=l, yshift=0.5cm] (s) {\color{red} $S^{miss}$}
            node[below of=y, xshift=-0.5cm] (label) {\textbf{(b)}}
            
            (x) edge[] (y)
            (x) edge[double, blue, bend left] (l)
            (y) edge[double, blue] (l)

            (x) edge[bend right=15] (s)
            (y) edge[bend right=15] (s)
            (d) edge[bend right=20] (x)
            (d) edge[bend left=25] (y)
            (d) edge[bend left=35] (r)
            (s) edge[bend right=15] (r)
            (d) edge[] (z)
            (z) edge[] (x)
            (z) edge[bend right] (y)
            (d) edge[bend right=15] (s)
            ;
        \end{scope}
    \end{tikzpicture}}

    \vspace{0em}
    \caption{\small m-DAGs showing (a) gradients are likely MNAR in FL, and (b)  assumptions for missing data correction in FLOSS.}
    \label{fig:m-dags}
\end{figure*}
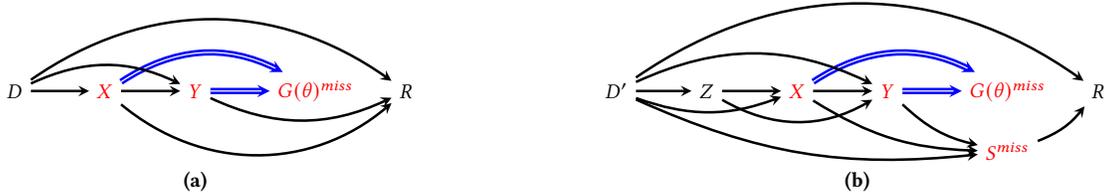

\section{A Formal Model of Missing Data in FL}

We use missing data directed acyclic graphs (m-DAGs) \citep{nabi2025sinica, mohan2021graphical} to provide a formal yet intuitive understanding of the impacts of missing data due to stragglers and user opt-out in FL systems, and to propose possible solutions.

An m-DAG $\mathcal G(V, E)$ is a DAG whose vertices $V$ correspond to random variables (or sets of random variables), some of which may be missing or completely unobserved, and whose edges $E$ encode substantive causal relations between these variables. In particular, the presence of a directed edge $A\rightarrow B$ implies that $A$ is a \emph{potential} cause of $B$ relative to other variables in $V$; the absence of such an edge implies that $A$ is \emph{not} a direct cause of $B$ relative to other variables in $V$.

The absence of edges in an m-DAG also encode statistical relations between the variables via the well-known d-separation criterion \cite{pearl1988probabilistic} defined as follows. A \emph{path} in an m-DAG $\mathcal G$ is an alternating sequence of vertices and edges $V_1 {\ \bf -\ } V_2 {\ \bf -\ } V_3, \dots,  V_K$, where each ``$-$'' in the sequence  is an edge $V_k \leftarrow V_{k+1}$ or $V_k \rightarrow V_{k+1}$ that exists in $\mathcal G$, and every vertex and edge in the sequence appears at most once. A vertex $V_k$ is said to be a {\it collider} on the path if the preceding and succeeding edges both point into it, {\it i.e.}, the path contains $V_{k-1} \rightarrow V_k \leftarrow V_{k+1}$. 
Given disjoint sets of vertices $A, B$ and $C$, the sets $A$ and $B$ are said to be d-separated given $C$, denoted $A \ci_{\text{d-sep}} B \mid C$, if and only there is no path from a vertex in $A$ to one in $B$ along which (i) every collider on the path is either in $C$ or has a descendant 
in $C$ and (ii) every non-collider on the path is not in $C$. 
Paths satisfying conditions (i) and (ii) are said to be \emph{open}. This definition leads to the following \emph{global Markov property} of m-DAGs---given disjoint sets $A, B, C$ we have, $A \ci_{\text{d-sep}} B \mid C \implies A \ci B \mid C \text{ in } p(V)$. That is, d-separation in $\mathcal G$ implies conditional independences in the probability distribution on the random variables displayed in the m-DAG. This gives us an intuitive way to reason about missingness in FL systems, as we now demonstrate.

\subsection{m-DAG Representation of Missingness in FL}
In Figure~\ref{fig:m-dags}(a), we propose an m-DAG relevant to our FL setup. We use red to mark variables that may be unobserved to the central server. In FL, the features $X$ and outcome $Y$ are completely unobserved to the central server so they are marked as red. Further, the gradients 
are also marked red, as they are missing for stragglers as well as any users who opt out of sharing their data; we use a superscript to distinguish this from the fully hidden case, as $G(\theta)^{miss}$. The blue edges $X {\color{blue} \implies} G(\theta)^{miss} {\color{blue} \impliedby} Y$ are used to highlight that the gradients are obtained as outputs of applying $h_\theta$ to $X, Y$, thus triggering the data-sharing agreement. All other variables---user info $D$, and the binary indicator $R$ denoting whether the central server is able to receive gradient data from a user device ($R=1$ for yes and $R=0$ for no)---are  fully observed.

We now justify why the missing gradients cannot be considered missing completely at random (MCAR) in FL systems. The data are considered MCAR if $R \ci G(\theta)^{miss}$ \cite{rubin1976inference}. In Figure~\ref{fig:m-dags}(a), we see that heterogeneity in user devices and demographics can influence the missingness indicator $R$, encoded by the edge $D\rightarrow R$. Further, $D$ can also influence the kinds of data $X, Y$ users process  on their device. Thus, we have a few open paths between $R$ and $G(\theta)^{miss}$---{\it e.g.}, $R \leftarrow D \rightarrow X {\color{blue} \implies} G(\theta)^{miss}$, implying $R \not\ci G(\theta)^{miss}$ by d-separation. These open paths must be blocked by adjusting for the covariates $D$ to mitigate bias resulting from aggregating gradients from only non-straggling devices. 

However, there is another complication arising from user opt-out that likely results in gradients that are missing not at random (MNAR). The data are MNAR if missingness is not independent of the missing variable given observed covariates alone, {\it i.e.}, $R \not\ci G(\theta)^{miss} \mid D$ in our FL setup. This occurs when user opt-out is influenced by the data $X, Y$ itself---{\it e.g.}, a user may not want to share interactions with the model $h_\theta$ involving sensitive data $X$ or if they are dissatisfied with model predictions of their outcomes $Y$---encoded by the edges $X\rightarrow R \leftarrow Y$. This leads to open paths---{\it e.g.}, $R \leftarrow Y {\color{blue} \implies} G(\theta)^{miss}$---that imply $R \not\ci G(\theta)^{miss} \mid D$.

Thus, we have established missing gradients in FL systems are likely to be MNAR. The following proposition formalizes how this degrades  FL accuracy if missingness is ignored.

\begin{proposition}
    Let $m < n$ denote the number of responsive devices. Model updates using only observed gradients do not approximate minimization of the true unobserved risk $E[L(\theta)^{miss}]$, even as $m\to\infty$ for the missingness in Figure~\ref{fig:m-dags}(a).
\end{proposition}
\begin{proof} (Sketch) Using  gradients from just observed devices is equivalent to solving an empirical risk minimization problem with risk $\frac{1}{m}\sum_{i=1}^n R_iL(x_i, y_i \theta)^{miss}$. As $m\to\infty$, this converges to $E[L(\theta)^{miss} \mid R=1]$, which is in general not equal to $E[L(\theta)^{miss}]$ when data are not MCAR, as in Figure~\ref{fig:m-dags}(a).
\end{proof}

That is, simply increasing the number of observed devices does not address the problem of systematic missingness in FL systems. We propose a solution for this in the next section.

\section{Reweighted Device Selection}
\label{sec:reweight}

Inverse probability weighting (IPW)---weighting observed cases by the inverse of their probability of being observed---is a common approach to unbiased estimation in missing data problems~\cite{horvitz1952generalization, seaman2013review}. Note that if missingness was only a function of device and user attributes $D$, we could estimate the required weights for IPW $1/p(R=1\mid D)$ using observed data alone. However, to keep our method as general as possible, we will allow for dependence on any of $X, Y, D$, which naturally allows for dependence on just $D$ as a special case. 

It is well known that unbiased inference with MNAR data is impossible without any assumptions \cite{rubin1976inference, nabi2025sinica}. Here, we will assume that the dependence of $R$ on $X$ and $Y$ is mediated by the user's (dis)satisfaction with their interactions with the system, {\it i.e.}, their willingness to share data is mediated by how well the model is performing at mapping their input features to outcomes. User satisfaction is typically already measured intermittently in modern FL applications via prompts of the kind shown in Figure~\ref{fig:teaser}(b). Note we make no assumptions about the functional form of dependence 
and instead estimate it from data. We also allow user satisfaction to be missing due to device unresponsiveness, or the user simply choosing not to provide  feedback. These assumptions are captured by the m-DAG in Figure~\ref{fig:m-dags}(b) with the addition of the variable $S^{miss}$ and associated edges.

\begin{algorithm}[t]
\caption{FLOSS Pseudocode}
\begin{algorithmic}[1]
\small
\State \textbf{Sign-up}: Record basic user info $D$ on  central server
\State \textbf{Initialize} $\theta^{(0)}$ (random or pre-trained) and broadcast it
\vspace{0.5em}
\For {each round/epoch of FL}
\State \textbf{Prompt} all users $u \in \mathcal U$ for participation, record $R$
\State \textbf{Prompt} all users $u \in \mathcal U$ for satisfaction, record $S^{miss}$
\State \textbf{Compute} ${\bf \pi} \coloneqq p(R=1\mid D', S^{miss})$ by solving \eqref{eq:shadow_equations}
\State Define ${\mathcal U}_R$ as users $u \in \mathcal U$ such that $R=1$
\vspace{0.5em}
\For{$i$ from $1$ to $max\ iterations$}
    \State \textbf{Weighted sampling} of $k$ users w/ replacement
    \Statex \hspace{1cm} from ${\mathcal U}_R$ using  1/$\pi$ as weights
    \State Locally compute gradients $g_1^{(t)}, \dots, g_k^{(t)}$
    \State \textbf{Upload} noisy, clipped gradients $\widetilde{g}_1^{(t)}, \dots, \widetilde{g}_k^{(t)}$
    \State \textbf{Timeout stragglers} after a fixed cutoff
    \State Aggregate non-straggler gradients to obtain $\overline{g}^{(t)}$
    \State \textbf{Update} $\theta^{(t+1)} \gets \theta^{(t)} - \eta \cdot \overline{g}^{(t)}$
    \State \textbf{Broadcast} $\theta^{(t+1)}$ to all users $u \in \mathcal U$
\EndFor
\EndFor

\vspace{0.5em}

\State \bf{return} $\theta^{(t)}$ from last update

\end{algorithmic}
\label{alg:floss}
\end{algorithm}

Under this model, we need to estimate the probability of missingness $p(R=1\mid D, S^{miss})$, which is still a function of missing variables corresponding to MNAR data. To make progress, say there is a variable $Z \in D$ such that (i) $Z \not\ci S^{miss} \mid R, D'$ and (ii) $Z \ci R \mid S^{miss}, D'$, where $D' = D\setminus \{Z\}$. Such a variable, known as a \emph{shadow variable} \cite{miao2024identification, chen2023causal}, is shown in Figure~\ref{fig:m-dags}(b). That is, $Z$ is a variable such as device processing power that might affect what kinds of data are processed on it, but does not necessarily drive missingness, which is instead affected by other device attributes in $D'$, such as network card specs determining network connectivity. With such a shadow variable,  it is possible to estimate $\pi \coloneqq p(R=1 \mid D', S^{miss})$ using results in \cite{miao2024identification, chen2023causal} by solving for parameters $\beta$ in a system of equations, where each equation is of the form
{\small
\begin{align}
    E\bigg[\bigg( \frac{R}{p_\beta(R=1\mid D', S^{miss})} - 1 \bigg) \cdot f_i(D', Z) \bigg]= 0,
    \label{eq:shadow_equations}
\end{align}}
and $f_i, \dots, f_q$ are any non-redundant functions of $D', Z$; more equations correspond to more complex parameterizations $\beta$  for $p(R=1\mid D', S^{miss})$. %
Note $R$ in the numerator of \eqref{eq:shadow_equations} ensures estimation usage of just observed data.

Using estimated probabilities $\pi$ from \eqref{eq:shadow_equations}, Proposition~\ref{prop:floss-works} formalizes that sampling clients with weights $1/\pi$ (rather than sampling uniformly at random) at each step of FL, does in fact minimize the true unobserved risk.
\begin{proposition}
    Under the assumptions of Figure~\ref{fig:m-dags}(b), model updates using gradients from observed devices obtained by weighted sampling using weights $1/\pi$ approximate minimization of the true unobserved risk $E[L(\theta)^{miss}]$  as $n\to\infty$.
    \label{prop:floss-works}
\end{proposition}
\begin{proof} (Sketch) This is equivalent to solving an empirical risk minimization problem with risk $\frac{1}{n}\sum_{i=1}^n \frac{R_iL(x_i, y_i \theta)^{miss}}{\pi_i}$, which converges to $E[\frac{R\cdot L(\theta)^{miss}}{\pi}]$ as $n\to\infty$. This in turn is equal to $E[L(\theta)^{miss}]$ under the assumptions of Figure~\ref{fig:m-dags}(b) \cite{chen2023causal, miao2024identification}.
\end{proof}

Pseudocode for our system FLOSS that incorporates this idea is shown in Algorithm~\ref{alg:floss}. The weighted sampling occurs in line~9, providing robustness to missingness that may occur in lines~4 and 12 due to user opt-out and stragglers respectively. The provided pseudocode also incorporates differentially private stochastic gradient descent, as in \cite{abadi2016deep}.

\begin{figure}[t]
    \centering
        \includegraphics[scale=0.5]{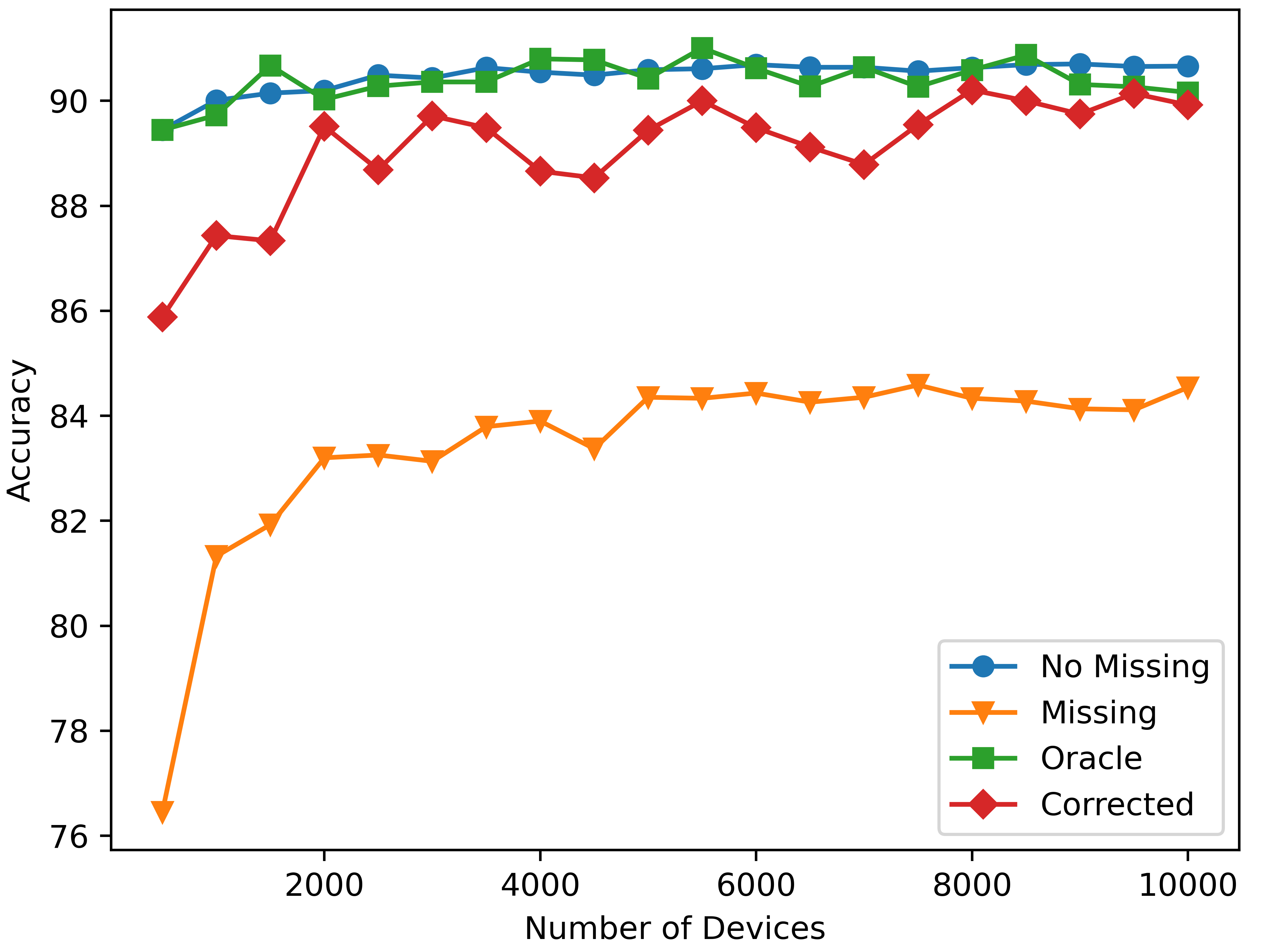}
\caption{\small Accuracy of FL with/without MNAR correction. \vspace{0em}}
\label{fig:results}
\end{figure}

\section{Preliminary Results \& Discussion}
\label{sec:experiments}

We implemented FLOSS in Python with a more robust implementation in Flower~\cite{flower} currently underway.  FLOSS runs in three different modes to simulate the effects of missingness. The server can run without missing data, where clients participate regardless of their response value $R$. It can run with missing data, where we allow clients to probabilistically opt out of training without any corrections. Finally, FLOSS can run with ``corrected'' missing data, where clients can probabilistically opt out, but we use correction techniques to mitigate the effects of the missing data.

We ran experiments to validate our theoretical results, as shown in Figure~\ref{fig:results}.  For differing numbers of simulated clients, we measure the average accuracy of a model trained on a binary classification task with no missing data (blue line), MNAR data (orange line), MNAR data with oracle correction (green line), and MNAR data with FLOSS (red line). The oracle correction assumes we know the true probability of a client opting out.  From these results, we conclude that not correcting for MNAR data negatively impacts the accuracy of the model, even for a relatively simple task. Additionally, we note that the correction from FLOSS closely mimics the no missing data case as we increase the number of clients. Further, adding more clients does not improve model accuracy unless missingness is taken into account, as seen by the gap in the orange and red lines. Thus, our results demonstrate that FLOSS is able to reduce the degradation of performance when MNAR data are present.

Now that we have a working prototype implementation, our next step is to run diverse experiments using our Flower implementation of FLOSS.  The Flower framework allows us to extensively evaluate the performance of FLOSS using larger datasets in more realistic settings.  We will also quantify the computational overhead of Algorithm 1 to show how feasible FLOSS is for real world deployments. Finally, we will experimentally show how FLOSS scales to evaluate the overall generalizability and accuracy of our approach.  

\section{Conclusion}
\label{sec:conclusion}
Though we discussed concepts from the perspective of supervised ML, they apply equally well to generative models.  While other assumptions may be possible for handling MNAR data in FL systems, our goal was to formalize the issues and provide an example framework with a plausible set of assumptions that can be built and expanded upon in future work. We have further demonstrated promising empirical results of our prototype system FLOSS, and hope this opens new areas of research into robust FL systems that tolerate real-world complications of missing data.

\begin{acks}
RB acknowledges support from the NSF CRII grant 2348287. The content of the information does
not necessarily reflect the position or the policy of the Government, and no official endorsement
should be inferred.
\end{acks}

\bibliographystyle{ACM-Reference-Format}
\bibliography{references}

\end{document}